\documentclass{article}

\usepackage{amssymb}
\usepackage{amsthm}
\usepackage{mathtools}
\usepackage{multirow}
\newtheorem{thm}{Theorem}
\newtheorem{prop}{Proposition}

\usepackage{graphicx}
\usepackage{subcaption}
\usepackage{url}

\newcommand{\eg}{e.\,g.,\ }
\newcommand{\ie}{i.\,e.,\ }
\newcommand{\Nat}{\mathord{\mathbb{N}}}
\newcommand{\tu}[1]{\stackrel{#1}{\to}}
\newcommand{\pre}[1]{{^\bullet{#1}}}
\newcommand{\post}[1]{{#1^\bullet}}
\newcommand{\reachable}[2]{\mathcal{R}(#1,#2)}
\newcommand{\Loop}{\mathit{Loop}}

\title{Discovering Hierarchical Process Models: an Approach Based on Events Clustering
\thanks{This work is supported by the Basic Research Program at HSE University, Russia.}}

\author{Antonina K. Begicheva, Irina A. Lomazova, Roman A. Nesterov}
\date{\normalsize HSE University, Faculty of Computer Science\\%
\normalsize12 Pokrovksy Boulevard, Moscow, 109028, Russia}

\begin{document}
\maketitle

\begin{abstract}
Process mining is a field of computer science that deals with  discovery and analysis of process models based on automatically generated event logs. 
Currently, many companies use this technology for  optimization and improving their processes. 
However, a discovered process model may be too detailed, sophisticated and difficult for experts to understand. 
In this paper, we consider the problem of discovering a hierarchical business process model from a low-level event log, i.e., the problem of automatic synthesis of more readable and understandable process models based on  information stored in  event logs of information systems. 

Discovery of better structured and more readable process models is intensively studied in the frame of process mining research from different perspectives. 
In this paper, we present an algorithm for discovering hierarchical process models represented as two-level workflow nets. The algorithm is based on  predefined event clustering  so that the cluster defines a sub-process corresponding to a high-level transition at the top level of the net. 
Unlike existing solutions, our algorithm does not impose  restrictions on the process control flow and allows for concurrency and iteration.

\end{abstract}



\section{Introduction}\label{sec1}

Over the past decade, companies whose processes are supported by various information systems have become convinced of the need to store as much potentially useful information about the execution of processes within the system as possible. This was facilitated by the qualitative development of areas related to the extraction of valuable information from recorded data, which helps to correct the work of organizations in time and, thus, save and increase their resources. Process mining is a field of science that provides a palette of tools for extracting the logic of system behavior, as well as modeling and optimizing the processes occurring in it. In particular, process mining methods allow you to find inconsistencies between the planned and actual behavior of the system, and also timely track the appearance of inefficient or incorrect behavior.

Despite the fact that more and more attention is being paid to preserving the optimal amount of necessary information about the execution of processes, process execution data is not always available in a convenient format and with the necessary degree of detail, since system logs are generated automatically. Process discovery is aimed at extracting processes from event logs and their representation in the form of a model.  Most of the available process discovery methods provide a model with the same level of detail as a given event log \cite{augusto2018automated}.

Because of this, a promising area for research is the task of discovering a more readable process model from a detailed event log, while preserving important information about the process execution for experts. Readability can be provided in various ways. The most commonly used methods are filtering rare behavior from the original event log, skipping "minor" events (the significance of an event is assessed according to the chosen methodology), and abstraction, where some events are considered indistinguishable from each other. In our study, we consider the latter approach, when more readable models are the result of model abstraction --- they are more compact and have an optimal level of detail for the work of experts than what could be obtained by  direct discovery methods.
In order not to lose important information, we are dealing not only with abstract (high-level) models, but also with hierarchical models that store low-level information in the form of sub-processes.

Thus, in this paper, we propose an algorithm for discovering hierarchical process models from event logs. Processes are modeled with workflow nets \cite{wfnets00}, a special subclass of Petri nets for modeling a control flow of business processes. This study continues our previous study 
\cite{begicheva2017discovering}, in which we have proposed an approach to discover abstract models for processes without cycles.
Here we provide a more general solution by overcoming the prohibition of  cyclic behavior. 

Hierarchical models allows us to have a high-level view of the model by ``folding''  the behavior of an individual  sub-process  into a high-level transition  with the ability to unfold it back. 
So, at the top level there is a high-level model, in which each individual transition corresponds to a sub-process built from low-level events.
The history of the detailed behavior of the process is recorded in a low-level log. 
Regarding the number of levels in the hierarchy, we will only use two levels --- high and low, but the algorithm can naturally be extended to any number of levels.

The paper is structured as follows. Section 2 presents the review of related research. 
Section 3 gives theoretical preliminaries and definitions used in the text. 
In Section 4, we discuss the basics of the hierarchical process discovery algorithm. 
Section 5 presents the main discovery algorithm and the proof of its correctness. 
Section 6 reports the outcomes from the experimental evaluation. 
In Section 7, we conclude the paper and discuss the possible future work directions. 


\section{Related Work}\label{sec2}

Research connected with our paper can be classified into approaches to abstracting event logs and process models and approaches to discovering hierarchical process models from event logs.

The recent work \cite{van2021event} gives a comprehensive review of approaches and methods that can be applied for low-level events abstraction.
The authors divide the methods according to: the learning strategy (supervised or unsupervised), the structure of the process models (strictly sequential or with interleaving), the low-level events grouping approach (deterministic or probabilistic), the nature of the processed data (discrete or continuous data). 
For instance, the method presented in \cite{maneschijn2022methodology} is a supervised method that converts a low-level event log into its high-level representation using detected behavioral patterns, such as sequence, selection, parallelism, interleaving and repetition.

A supervised approach to event abstraction was also presented in \cite{Mannhardt2016}.
This method takes a low-level event log and transforms it to an event log at the desired level of abstraction, using behavioral patterns: sequence, choice, parallel, interleaving and repetition of events.
This technique allows us to obtain a reliable mapping from low-level events to activity patterns automatically and construct a high-level event log using these patterns.
Another supervised event abstraction method was discussed in \cite{tax2016event}. 
The nature of this method is as follows.
We annotate a low-level event with the correct high-level event using the domain knowledge from the actual process model by the special attribute in an event log. 
In addition, this paper assumes that multiple high-level events are executed in parallel. 
This allows us to interpret a sequence of identical values as the single instance of a high-level event. 

A general approach to the representation of multi-level event logs and corresponding multi-level hierarchical models was studied in \cite{leemans2020using}.
The authors especially note that this approach can combine multiple formalisms for representing different levels in a multi-level process models.

There are many ways of abstracting process models by reducing their size in order to make them more convenient to work with. 
Each method may be useful depending on a group of interrelated factors: the abstraction purposes, the presence of certain patterns and constructs, and the specifics of modeling notation. 
Reducing the size of the model by abstraction can be done as the ``convolution'' of groups of elements, or implemented by throwing some parts away (insignificant in a particular case). 
The importance of the low-level event log abstraction is emphasized, among the others, in \cite{Smirnov2012}. 

The researchers determine, which level of abstraction is appropriate for a particular case in different ways, but the main criterion is that the model should be readable and understandable. 
In \cite{gunther2007fuzzy}, the abstraction of a process model occurs through ``simplification'' automatically: a user determines only the desired degree of detail, but not the actual correctness of identifying high-level events.
Conversely, the paper \cite{maneschijn2022methodology} stressed the importance of the abstraction level dependence on the domain expert knowledge.

Petri nets \cite{Reisig13} can also be extended by adding the hierarchy as, e.g., in Colored Petri nets (CPN) \cite{jensen2009coloured}.
Hierarchy also allows one to construct more compact, readable and understandable process models. 
Hierarchy of CPN models can be used as an abstraction, in the case of two levels: a high-level \emph{abstract} model and a low-level \emph{refined} model. 
In our paper the high-level model is a model with abstract transitions. 
An abstract transition refers to a Petri net subprocess, which refines the activity represented by this high-level transition. 
The complete low-level (``flat'') process model can be obtained from a high-level model by substituting subprocesses for high-level transitions.

``Flat'' synthesis (when process model and log are at the same-level) is a standard process discovery problem, which has been extensively studied in the literature.
A wide range of process discovery algorithms supports the automated flat process model synthesis \cite{augusto2018automated}.
\emph{Inductive miner} \cite{Leemans13} is one of the most widely used process discovery algorithm that produces well-structured process models recursively built from building blocks for standard behavioral patterns.
They can be potentially used for constructing high-level process models.
However, this technique does not take the actual correspondence between low-level events and subprocesses.
In \cite{greco2008mining}, the authors also used the recognition of behavioral patterns in a process by a structural clustering algorithm and then define a specific workflow schema for each pattern. 

In \cite{li2010mining}, a two-phase approach to mining hierarchical process models was presented. 
Process models here were considered as interactive and context-depen\-dent maps based on common execution patterns. 
On the first phase, an event log is abstracted to the desired level by detecting relevant execution patterns. 
An example of such pattern is the maximal repeat that captures typical sequences of activities the log. 
Every pattern is then estimated by its frequency, significance, or some other metric needed for accurate abstraction. 
On the second phase, Fuzzy miner discovery algorithm \cite{gunther2007fuzzy} adapted to process maps discovery is applied to  the transformed log.

\emph{FlexHMiner} \cite{lu2020discovering} is a general algorithm based on process trees implemented in ProM software. 
The authors stresses the flexibility of this approach: to identify the hierarchy of events, the method supports both supervised methods and methods using the general knowledge on a process. 
The limitations of this method include the fact that each of the sub-processes can be executed only once, which means that the method is not suitable for processes with cycles. 

Detecting high-level events based on the patterns of behavior in an event log does not make it possible to refine the accuracy of abstraction based on the general knowledge of the system or provide it only partially.
Patterns provide only the ability to change the scale, but not to participate in the selection of correct high-level events. 
This could be useful only for a superficial analysis.
However, there is a real risk of combining unrelated low-level events into a single high-level event only because they are executed sequentially, but not because they belong to the same logical component of a system.

A large amount of literature is devoted to the problem of discovering structured models from event logs.
Researchers offer different techniques to improve the structure of discovered models, e.g., in \cite{Fstr07}, and to produce already well-structured process models \cite{MpPM14,Pedro16}.
Different ways of detecting subprocesses in event logs using low-level transition systems were discussed in \cite{LocEv15,Kalenkova14-1,Kalenkova14}.
These papers do not consider mining hierarchical process models from event logs.

In our previous paper \cite{begicheva2017discovering}, we presented an algorithm for the discovery of a high-level process model from the event log for acyclic processes. 
The method takes the initial data about abstraction in the form of a set of detailed events grouped into high-level ones, which means that any method of identifying abstract events can potentially be used, including those based on expert knowledge. 
Moreover, the algorithm is based on event logs pre-processing.
This includes converting an initial event log into a high-level form and improving a target high-level process model
After pre-processing, we can use any of the existing process discovery algorithms suitable for the same-level (flat) process discovery. 
Both possibilities of using existing approaches as components make the earlier proposed algorithm flexible.

This paper expands the conditions of the applicability of the algorithm from \cite{begicheva2017discovering} since  it only works for acyclic models. 
For the algorithm to find and process potential cycles in the event log, we will reuse the method for detecting the repetitive behavior in a event log proposed in \cite{alvarez2020identifying}, which partially covers the general solution of the cycle detection problem.

\section{Preliminaries}\label{sec3}

By $\Nat$  we denote the set of non-negative integers.

Let $X$ be a set. A \emph{multiset} $m$ over a set $X$ is a mapping: $m : X \rightarrow \Nat$, where $\Nat$ -- is the set of natural numbers (including zero), i.e., a multiset may contain several copies of the same element. 
For an element $x\in X$, we write  $x\in m$ , if $m(x) >0$.
For two multisets $m,m'$ over $X$ we write $m\subseteq m'$  iff $\forall x\in X: m(x) \leq m'(x)$ (the inclusion relation). 
The sum, the union and the subtraction of two multisets $m$ and $m'$ are defined as usual:
$\forall x\in X: (m+m')(x)=m(x)+m'(x), (m\cup m')(x)=max(m(x),m'(x)), (m-m')(x)=m(x) - m'(x)$, if  $m(x) - m'(x) \geq 0$, otherwise $(m-m')(x)= 0$.
By $\mathcal{M}(X)$ we denote the set of all multisets over $X$.

\smallskip

For a set $X$, by $X^*$ with elements of the form $\langle x_1,\dots, x_k \rangle$ we denote the set of all finite sequences (words) over $X$, $\langle \rangle$ denotes the empty word, \ie the word of zero length. Concatenation of two words $w_1$ and $w_2$ is denoted by $w_1\cdot w_2$.

Let  $Q \subseteq X$ be a subset of $X$. The 
 projection $\restriction_Q : X^*  \rightarrow Q^*$ is  defined recursively as follows: 
 ${\langle \rangle}{\restriction_Q}  = \langle \rangle$,  and for $\sigma \in X^*$ and $x \in X$: 
$$(\sigma	\cdot \langle x\rangle) {\restriction_Q}= 
    \begin{cases} 
    \sigma {\restriction_Q}           $ if $ x\notin Q \\
   \sigma {\restriction_Q}  \cdot  \langle x \rangle $ if $ x \in Q
    \end{cases} $$

\smallskip

We say that $X = X_1 \cup X_2 \cup \dots  \cup X_n$ is a partition of the set $X$ if for all $1{\leq}i,j{\leq}n $ such that 
$i\not = j$ we have $X_i \cap X_j = \emptyset$. 

\subsection{Petri Nets}
   

Let $P$ and $T$  be two disjoint finite sets of places and transitions respectively, and $F : (P \times T) \cup (T \times P) \to \Nat$ be an arc-weight function. 
Let also $A$ be a finite set of \emph{event names} (or \emph{activities}) representing observable actions or events, $\tau$ --- a special label for \emph{silent} or invisible action, $\lambda : T \to A\cup \{\tau\}$ is a transition labeling function. 
Then $N = (P,T,F, \lambda)$ is a \emph{labeled Petri net}. 

Graphically, a Petri net is designated as a bipartite graph, where places are represented by circles, transitions by boxes, and the flow relation $F$ by directed arcs.


A \emph{marking} in a Petri net $N = (P,T,F, \lambda)$ is a function $m : P \to \Nat$  mapping  each place to some  number of tokens (possibly zero). Hence, a marking in a Petri net may be considered as a multiset over its set of places.  
Tokens are graphically designated by filled circles, and then a current marking $m$  is  represented by putting $m(p)$ tokens into each place $p \in P$.
A \emph{marked Petri net} $(N,m_0)$ is a Petri net $N$ together with its initial marking $m_0$.

For  transition $t \in T$, its \emph{preset}  (denoted $\pre{t} $)  and its \emph{postset} (denoted $\post{t}$) are defined as  sets of its input and output places respectively, \ie 
$\pre{t} = \{p \mid F(p,t)\not = 0\}$ and $\post{t}= \{p \mid F(t,p)\not = 0\}$.

A transition $t \in T$ is \emph{enabled} in a marking $m$, if for all $p \in \pre{t}$, $m(p)  \geq F(p,t)$. An enabled transition $t$ may \emph{fire} yielding a new marking $m'$, such that $m'(p) = m(p) - F(p,t) + F(t,p)$ for each $p \in P$ (denoted $m\tu{\lambda(t)} m'$, or just $m\to m'$).
A marking $m'$ is reachable from a marking $m$, if there exists a sequence of firings $m=m_0 \to m_1 \to \dots m_k=m'$. By 
$\reachable{N}{m}$ we denote the set of all markings reachable from marking $m$ in a net $N$.

Let $(N,m_0)$ be a marked Petri net with transitions labeled by activities from $A\cup\{\tau\}$, and let $m_0 \tu{a_1} m_1 \tu{a_2} \dots$ be a finite or infinite sequence of firings in $N$, which starts from the initial marking $m_0$ and cannot be extended. Then a sequence of observable activities $\rho$, such that $\rho = 
\langle a_1, a_2,\dots \rangle{\restriction_A}$, is called a \emph{run}, \ie a run is a sequence of observable activities representing a variant of Petri net behavior. For a finite run $\rho$, which corresponds to a sequence of firings $m_0\tu{a_1}\dots \tu{a_k}m_k$, we call $m_0$ and $m_k$ its initial and final markings respectively.

A transition $t\in T$ is called \emph{dead} for a marked net $(N,m_0)$, if for each reachable marking $m\in \reachable{N}{m_0}$, $t$ is not enabled in $m$.




\begin{figure}[ht]
 
    \centering
    \includegraphics[width=0.7\textwidth]{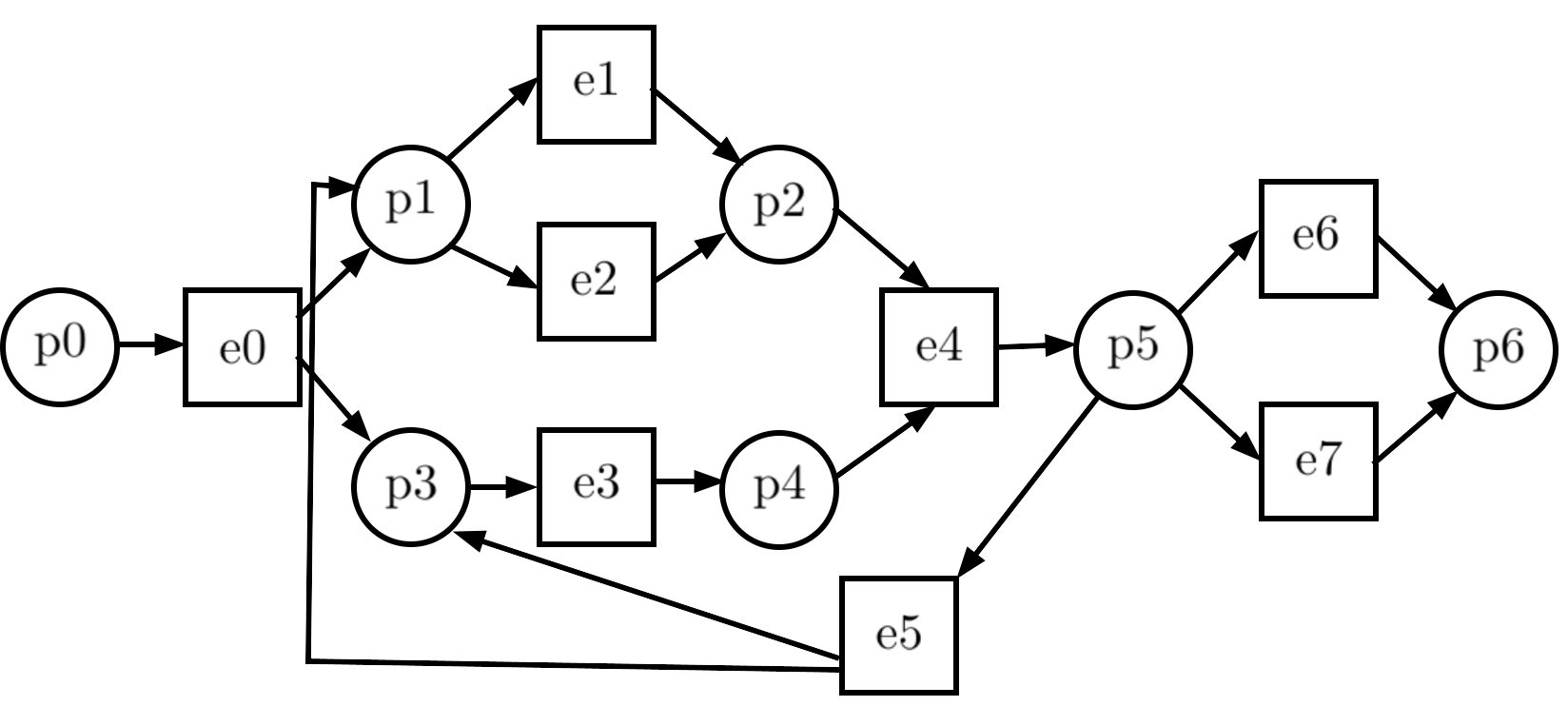}
    \caption{A workflow net for handling compensation requests}   
    \label{fig:petrinet}
\end{figure}

\smallskip

In our study we consider  \emph{workflow nets} --- a special subclass of Petri nets 
\cite{van2002workflow} 
for workflow modeling. A \emph{workflow net} is a (labeled) Petri net with two special places: $i$ and $f$. These places mark the beginning and the ending of a workflow process.

A (labeled) marked Petri net $N = (P,T,F,\lambda, m_0)$ is called a workflow net (WF-net) if the following conditions hold:
\begin{enumerate}
\item	There is one source place $i\in P$ and one sink place $f\in P$, such that  $\pre{i} = \post{f} = \emptyset$.
\item	Every node from $P \cup T$ is on a path from $i$ to $f$.
\item	The initial marking $m_0$ in $N$ contains the only token in its source place.
\end{enumerate}

Given a WF-net, by $[i]$ we denote its initial marking with the only token in place $i$, and by $[f]$ --- its \emph{final} marking with the only token in place $f$.

 An example of a workflow net that simulates a simple process of handling ticket refund requests, is shown in Fig.~\ref{fig:petrinet} \cite{Aalst11}. Here $p_0$ is the source place, and $p_6$ --- the sink place. 

\smallskip

Soundness \cite{van2002workflow} is the main correctness property for workflow nets.
A WF-net $N = (P,T,F,\lambda, [i])$ is called \emph{sound}, if 
\begin{enumerate}
    \item For any marking $m \in R (N,[i])$, $[f]\in \reachable{N}{m}$;
    \item If for some $m \in R (N,[i])$, $[f]\subseteq m$, then $m=[f]$;
    \item There are no dead transitions in $N$.
\end{enumerate}



\subsection{Event Logs}

Most information systems record  history of their process execution into event logs. An \emph{event record} usually contains case ID, an activity name, a time step, and some information about resources, data, etc. For our study, we use case IDs for splitting an event log into  traces, timestamps --- for ordering events within each trace, and abstract from event information different from event names (activities).

Let A be a finite set of activities. A \emph{trace} $\sigma$ is a finite sequence of activities from $A$, i.e., $\sigma \in A^*$. By $\# a(\sigma)$ we denote the number of occurrences of activity $a$ in  trace $\sigma$. 

An \emph{event log} $L$ is a finite multi-set of traces, i.e., $L \in\mathcal{M}(A^*)$. 
A log $L'$ is called  a \emph{sub-log} of $L$, if $L'\subseteq L$.
Let $X \subseteq A$. We extend projection $\restriction_X$ to event logs, i.e., for an event log $L \in M(A^*)$, its projection is a sub-log $L{\restriction_X}$, defined as the multiset of projections of all traces from $L$, \ie $L{\restriction_X}(\sigma{\restriction_X}) = L(\sigma)$ for all $\sigma\in L$.

An important question is whether the event log matches the behavior of the process model and vice versa. There are several metrics to measure  conformance between a WF-net and an event log. Specifically, \emph{fitness} defines to what extend the log can be replayed by the model. 

  Let $N$ be a WF-net with transition labels from $A$, an initial marking $[i]$, and a final marking $[f]$. Let $\sigma$ be a trace over $A$. We say that trace $\sigma = \langle a_1, \dots , a_k\rangle$ \emph{perfectly fits} $N$, if $\sigma$ is a run in $N$ with initial marking $[i]$ and final marking $[f]$.
A log $L$ perfectly fits $N$, if every trace from $L$ perfectly fits $N$.

\section{Discovering Hierarchical WF-Nets}
\subsection{Hierarchical WF-Nets}
\label{subsec:HWF}


Let $A$ denote the set of low level process activity names.
Let $\Tilde{A} = \{\alpha_1, \alpha_2, \dots,$ $\alpha_k\}$ denote the set of sub-process names, which represent high-level activity names, respectively.

Here we define \emph{hierarchical workflow} (HWF) nets with two levels of representing the process behavior.
The \emph{high level} is a WF-net, whose transitions are labeled by the names of sub-processes from $\Tilde{A}$.
The \emph{low level} is a set of WF-nets corresponding to the \emph{refinement} of transitions in a high-level WF-net.
Transitions in a low level WF-net are labeled by the names of activities from $A$.
Below we provide the necessary definitions and study the semantics of HWF-nets.

An HWF-net is a tuple $\mathcal{N} = (\Tilde{N}, N_1, N_2, \dots, N_k,\ell)$, where:
\begin{enumerate}
    \item $\Tilde{N} = (\Tilde{P}, \Tilde{T}, \Tilde{F}, \Tilde{\lambda}, \Tilde{[i]})$ is a high-level WF-net, where $\Tilde{\lambda} \colon \Tilde{T} \to \Tilde{A}$ is a bijective transition labeling function, which assigns sub-process names to transitions in $\Tilde{N}$;
    \item $N_i = (P_i, T_i, F_i, \lambda_i, [i]_i)$ is a low-level WF-net for every $i \in [1, k]$ with a transition labeling function $\lambda_i \colon T_i \to A_i$, where $A_i \subseteq A$ is a  subset of low level activities for $N_i$;
    \item $\ell \colon \Tilde{A} \to \{N_1, N_2, \dots, N_k\}$ is a bijection which establish the correspondence between a sub-process name (transition in a high-level WF-net) and a low-level WF-net.
\end{enumerate}

Accordingly, every transition in a high-level WF-net $\mathcal{N}$ has a corresponding low-level WF-net modeling the behavior of a sub-process.
The example of an HWF-net is shown in Figure \ref{fig:hwf1}.
We only show the refinement of two transitions $t_1$ and $t_2$ in the high-level WF-net $\Tilde{N}$ with two low-level WF-nets $N_1$ and $N_2$.
They represent the detailed behavior of two sub-processes $\alpha_1$ and $\alpha_2$ correspondingly.

\begin{figure}[h]
        \centering
     \includegraphics[width=0.7\textwidth]{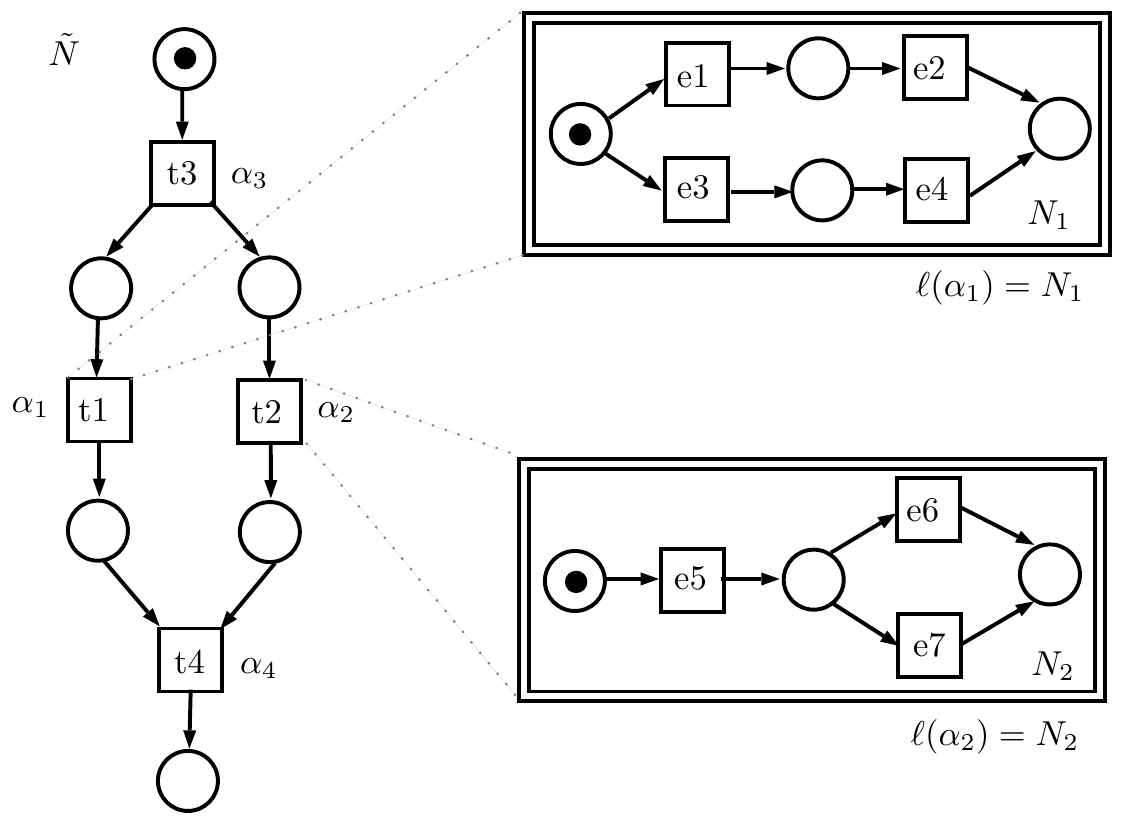}
     \caption{A HWF-net with two refined transitions}
    \label{fig:hwf1}
\end{figure}

We next consider the operational semantics of an HWF-nets by defining its run.
For what follows, let $\mathcal{N} = (\Tilde{N}, N_1, N_2, \dots, N_k, \ell)$ be an HWF-net.

Let $\Tilde{m}$ be a reachable marking in a high-level WF-net of an HWF-net $\mathcal{N}$ and $T_{\Tilde{m}}$ be a set of transitions enabled at $\Tilde{m}$.
Intuitively, a set of transitions enabled in a high-level WF-net $\Tilde{m}$ corresponds to a set of sub-processes for which we can start to fire their low level transitions.

If some transitions in a high-level WF-net enabled at $\Tilde{m}$ share common places, then there is a \emph{conflict}, and we can choose, which sub-process to start, while the other sub-rocesses corresponding to conflicting transitions in a high-level WF-net will not be able to run.

If some transitions in a high-level WF-net enabled at $\Tilde{m}$ do not share common places, then they are enabled \emph{concurrently}, and we can start all sub-processes corresponding to these concurrently enabled transitions using the ordinary interleaving semantics.

The firing of a transition in a high-level WF-net will be complete when the corresponding sub-process reaches its final marking.

For instance, let us consider the HWF-net shown in Figure \ref{fig:hwf1}.
After firing  high-level transition $t_3$ and executing a corresponding sub-process $\alpha_3$ (not provided in Figure \ref{fig:hwf1}), two high-level transitions $t_1$ and $t_2$ become enabled.
They do not share common places, i.e.,  high-level transitions $t_1$ and $t_2$ are enabled concurrently.
Thus, the corresponding sub-processes $\alpha_1$ (low-level WF-net $N_1$) and $\alpha_2$ (low-level WF-net $N_2$) can also be executed concurrently.
We can obtain a sequence $\rho = \langle \alpha_3, e_1, e_5, e_2, e_6, \alpha_4 \rangle$, which will represent a possible run of an HWF-net from Figure \ref{fig:hwf1}.
High-level activities $\alpha_3$ and $\alpha_4$ should also be replaced with corresponding sub-process runs.

Lastly, we give a straightforward approach to transforming an HWF-net $\mathcal{N} = (\Tilde{N}, N_1, N_2, \dots,$ $ N_k, \ell)$ to the corresponding \emph{flat} WF-net denoted by $\mathbf{fl}(\mathcal{N}) = (P, T, F, \lambda, [i])$.
We need to replace transitions in a high-level WF-net with their sub-process implementation given by low-level WF-net corresponding by $\ell$.
When a transition $t$ in a  high-level WF-net $\Tilde{N}$ is replaced by a low-level WF-net $N_i$, we need to fuse a source place in $N_i$ with all input places of $t$ and to fuse a sink place in $N_i$ with all output places of $t$.

For instance, the flat WF-net $\mathbf{fl}(\mathcal{N})$ constructed for the HWF-net, shown in Figure \ref{fig:hwf1}, is provided in Figure \ref{fig:flatwf}.
We replaced transition $t_1$ with $N_1$ and transition $t_2$ with $N_2$ as determined by the labels of low-level WF-nets.
This figure also shows the double-line contours of corresponding  high-level transitions.

\begin{figure}[!h]
        \centering
    \includegraphics[width=0.35\textwidth]{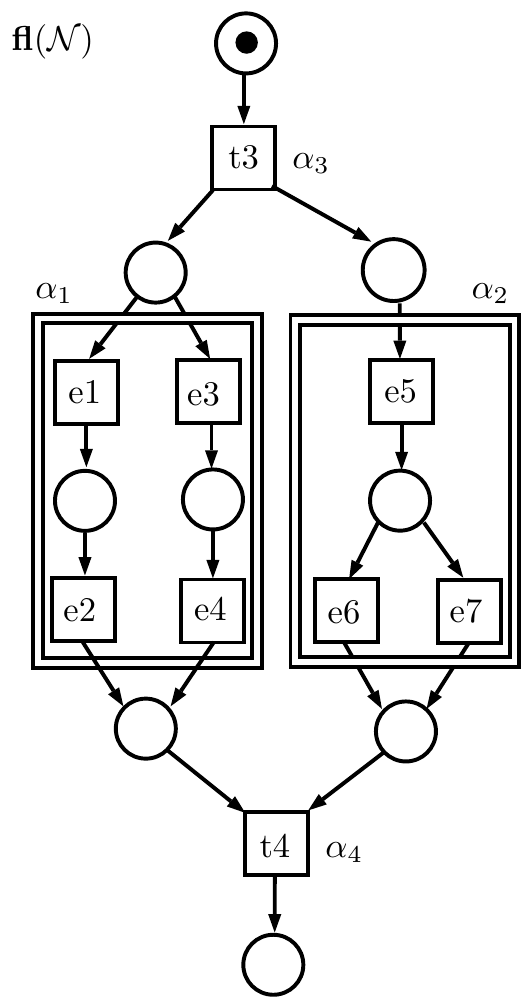}
    \caption{The flat WF-net for the HWF-net in Fig.~\ref{fig:hwf1}}
    \label{fig:flatwf}
\end{figure}

Proposition \ref{prop:conn} gives the main connection between an HWF-net and its flat representation. 

\begin{prop}\label{prop:conn}
    Let $\mathcal{N} = (\Tilde{N}, N_1, N_2, \dots, N_k, \ell)$ be a HWF-net, and $\mathbf{fl}(\mathcal{N})$ be the corresponding flat WF-net.
    A sequence $\rho$ is a run in $\mathcal{N}$ if and only if $\rho$ is a run in $\mathbf{fl}(\mathcal{N})$.
\end{prop}

In other words, the set of all possible runs of a HWF-net is exactly the same as the set of all possible runs of the corresponding flat WF-net.
Proof of this proposition directly follows from the construction of the flat WF-net and from the way we define the sequential semantics of a hierarchical WF-net.

To sum up, the constructive representation of the HWF-net sequential semantics fully agrees with the ordinary Petri net firing rule and the definition of a run discussed in Section 3.

\subsection{Problem Statement}
Let $L$ be a log over a set $A$ of activities, and let $A=A_1\cup A_2 \cup\dots\cup A_k$ be a partition of $A$. Let also $\Tilde{A} = \{\alpha_1, \alpha_2,\dots\alpha_k\}$ be a set of high-level activities (sub-process names).

The problem is to construct a HWF-net $\mathcal{N}= (\Tilde{N},N_1,N_2,\dots,N_k,l)$, where for each $i\in [1, k]$, $N_i$ is a sub-process (WF-net) labeled by $\alpha_i$  over the set of activities $A_i$. Runs of $\mathcal{N}$ should conform to traces from $L$.

We suppose that partitioning $A$ into subsets $A_1, \dots A_k$ is made either by an expert, or automatically based on some information contained in extended action records, such as resources or data. In Section~\ref{evaluation} we give two examples of partitioning activities for a real log. Then we consider that a sub-process is defined by its set of activities, and we suppose that sets of activities for two sub-processes do not intersect. If it is not the case and two sub-processes include some common activities like 'close the file', one can easily distinguish them by including resource or file name into activity identifier.

One more important comment concerning partitioning activities: we suppose that it does not violate the log control flow. Specifically, if there are iterations in the process, then for a set of iterated activities $B$ and for each sub-process activities set $A_i$, we suppose that either $B\cap A_i = \emptyset$, or $B\subseteq A_i$, or $A_i\subseteq B$. Note that this is a reasonable requirement, taking into account the concept of a sub-process. If still it is not true, 
\ie only a part of $A_i$ activities are iterated, then the partition can be refined,  exactly $A_i$ can be splitted into two subsets, within and out of iteration.

\subsection{Basics of the Proposed Solution}
Now we describe the main ideas and the structure of the algorithm for discovery of hierarchical WF-net from event log.

Let once more $L$ be a log with activities from $A$, and let $A=A_1\cup A_2 \cup\dots\cup A_k$ be a partition of $A$. Let also $\Tilde{A} = \{\alpha_1, \alpha_2,\dots\alpha_k\}$ be a set of high-level activities (sub-process names).
A hierarchical WF-net $\mathcal{N}$ consists of a high-level WF-net $\Tilde{N}$ with activities from $\Tilde{A}= \{\alpha_1,\dots,\alpha_k\}$, and $k$  sub-process WF-nets $N_1,N_2,\dots,N_k$, where for each $N_i$, all its activities belong to $A_i$.

Sub-process WF-nets $N_1,N_2,\dots,N_k$ can be discovered directly. To discover $N_i$ we filter log $L$ to $L_i = L{\restriction_{A_i}}$.
Then we apply one of popular algorithms (\eg Inductive Miner) to discover WF-model from the log $L_i$. Fitness and precision of the obtained model depend on the choice of the discovery algorithm.

Discovering high-level WF-model is not so easy and is quite a challenge. Problems with it are coursed by possible interleaving of concurrent sub-processes and iteration. A naive solution could be like follows: in the log $L$ replace each activity $a\in A_i$ by $\alpha_i$ --- the name of the corresponding sub-process. Remove 'stuttering', \ie replace wherever possible several sequential occurrences of the same high-level activity by one such activity. Then apply one of popular discovery algorithms to the new log over the set $\Tilde{A}$ of activities. However, this does not work.

\begin{figure}[h]
    \centering
    \begin{subfigure}{0.5\textwidth}
		     \includegraphics[width=1\textwidth]{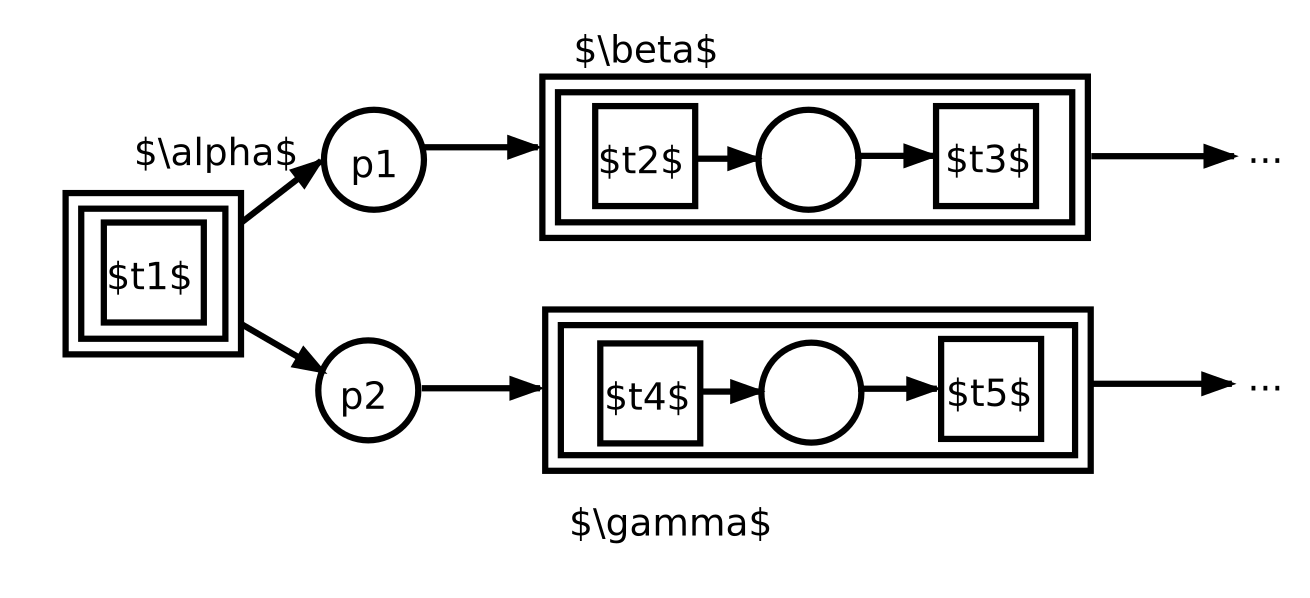}		
        \caption{concurrent sub-processes}
        \label{fig:interleaving}
    \end{subfigure}\\
    \begin{subfigure}{0.9\textwidth}
        \includegraphics[width=1\textwidth]{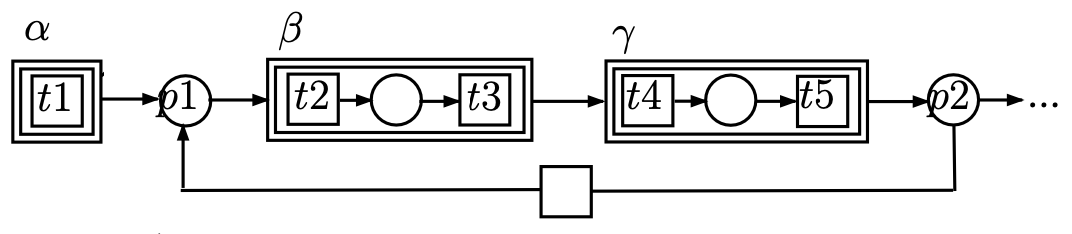}
        \caption{sub-processes inside a loop}
        \label{fig:cycles}
    \end{subfigure}

    \caption{Interleaving and iteration of sub-processes}
    \label{fig:subproc_problems}
\end{figure}

Consider examples in Fig.~\ref{fig:subproc_problems}. Fragment (a) in Fig.\ref{fig:subproc_problems} shows two concurrent sub-processes $\beta$ and $\gamma$, going after sub-process $\alpha$, which consists of just one transition. After replacing of low level activities with the corresponding sub-process names and removing stuttering, for the fragment (a), we get runs $\langle \alpha,  \beta,  \gamma, \dots\rangle$,  $\langle \alpha, \gamma, \beta, \dots\rangle$, $\langle \alpha, \beta,\gamma, \beta, \gamma, \dots\rangle$,
$\langle \alpha, \gamma, \beta, \gamma,  \beta,\dots\rangle$, etc.  
Fragment (b) in Fig.~\ref{fig:subproc_problems} shows a cycle. The body of this cycle is the sequence of two sub-processes $\beta$ and $\gamma$. Among runs for the fragment (b) we also have $\langle \alpha,  \beta,  \gamma, \dots\rangle$,  $\langle \alpha, \beta,\gamma, \beta, \gamma, \dots\rangle$. So, iterations should be considered separately. 

Discovering high-level WF-nets for acyclic models (\ie logs without iteration) was studied earlier in \cite{begicheva2017discovering}, where all details can be found. Here we call  this algorithm as Algorithm $\mathfrak{A}_0$ and illustrate it with the  example in Fig.~\ref{fig:subproc_problems}(a). The algorithm $\mathfrak{A}_0$ to discover a high-level WF-model from a log $L$ without iterations reduces this problem to the classical discovery problem, which can be solved by many popular algorithms, such as e.g. Inductive Miner.  Therefore, we parameterize Algorithm $\mathfrak{A}_0$ with Algorithm~$\mathfrak{D}$ for solving classical discovery problem.


\smallskip

\textit{Algorithm $\mathfrak{A}_0(\mathfrak{D})$}:

\begin{enumerate}
    \item[\textit{Step 1.}] For all traces in $L$, replace low level activities with the corresponding sub-process names and remove stuttering.
    \item[\textit{Step 2.}]  For each trace $\sigma$ with more than one occurrence of the same activity  replace $\sigma$ with the set of all possible clones of $\sigma$ by removing for each activity $\alpha$, all its multiple occurrences except one, and by removing (newly formed) stuttering. For example, the trace $\langle \alpha, \beta,\gamma, \beta, \gamma, \dots\rangle$ will be replaces by two traces $\langle \alpha, \beta,\gamma, \dots\rangle$ and $\langle \alpha, \gamma, \beta, \dots\rangle$ obtained by keeping the first occurrences of $\beta$ and $\gamma$, and correspondingly by keeping the first occurrence of $\gamma$ and the second occurrence of $\beta$. In this example, constructing clones by keeping other occurrences of $\gamma$ does not generate new traces.
    \item[\textit{Step 3.}]  Let $\Tilde{L}$ be the log resulting from executing two previous steps.  To obtain a high-level WF-net $\Tilde{N}$, apply a popular algorithm $\mathfrak{D}$ to discover a WF-net from event log $\Tilde{L}$.
\end{enumerate}

It was proved in~\cite{begicheva2017discovering} that if an algorithm used in Step~3 of the Algorithm $\mathfrak{A}_0$ for each input log $L$ discovers  a WF-net  perfectly fitting $L$, then the Algorithm $\mathfrak{A}_0$, given a log $L$ without repetitive behavior, produces a HWF-net $\mathcal{N}$, such that $\mathbf{fl}(\mathcal{N})$ perfectly fits $L$.

\medskip

Now we come to logs with repetitive behavior. The main idea here is to represent a loop body as a subset of its activities. Then a body of a loop can be considered as a sub-process with a new loop sub-process name.

To discover repetitive behavior we use  methods from \cite{alvarez2020identifying,tapia2017discovering}, which allow  to determine causal, concurrency, and repetitive relations between events in an event log.
Actually, for our purpose we need only repetitive relations and based on them loop discovery. The algorithm in \cite{tapia2017discovering} (we call it here as  Algorithm $\mathfrak{B}$) allows us to discover elementary loop bodies as sets of their activities and process them recursively, starting with inner elementary loops.  Thus, at every iteration we deal with a loop body without inner loops. To obtain a sub-trace, corresponding to a loop body with a set of activities $B$ from a log trace $\sigma$ we construct the projection $\sigma{\restriction_B}$. After filtering all current traces this way, we get an event log for discovering a WF-net modeling the loop body behavior by applying Algorithm $\mathfrak{A}_0$.

Then the result high-level WF-net is built recursively by substituting discovered loop bodies for loop sub-process  names, starting from inner loops.

\section{Algorithm for Discovering  HWF-Nets from Low Level Event Logs}\label{sec5}

Here we describe  our discovery algorithm  in more details.

Let $A$ be a set of activities and $L$ --- a log over $A$. Let then $A=A_1\cup\dots\cup A_k$ be a partition of $A$, where for $i\in [1, k]$, $A_i$ is a set of activities of a sub-process $\alpha_i$ and $\Tilde{A}= \{\alpha_1, \dots, \alpha_k\}$ --- a set of sub-process names.

Then Algorithm $\mathfrak{A}(\mathfrak{D})$  constructs a HWF-net $\mathcal{N} = (\Tilde{N}, N_1, N_2, \dots, N_k,\ell)$ with  high-level WF-net $\Tilde{N} = (\Tilde{P}, \Tilde{T}, \Tilde{F}, \Tilde{\lambda}, \Tilde{[i]})$, where
$\Tilde{\lambda} \colon \Tilde{T} \to \Tilde{A}$ and for each $\alpha_i\in \Tilde{A}$, $\ell(\alpha_i) = N_i$, \ie sub-process name $\alpha_i$ corresponds to  low-level WF-net $N_i$ in $\mathcal{N}$.

\smallskip

\textit{Algorithm $\mathfrak{A}(\mathfrak{D})$}: 

\smallskip
By $\Tilde{B}$ we denote a  set of loop body names with elements $\{\beta, \beta',\dots\}$ and by $\ell_B$ --- a function which maps a name from $\Tilde{B}$ to a WF-net that implements the loop body. For a WF-net $N$, denote by $\Loop(N)$  a WF-net that is a loop with  body $N$.

\begin{enumerate} 
\item[\textit{Step 0.}] Set $\Tilde{B}= \emptyset$ and $\ell, \ell_B$ to be two functions with  empty domains. 
\item[\textit{Step 1.}]
    Apply Algorithm $\mathfrak{B}$ to $L$ to find a set $B$ of activities of some inner elementary loop body. If there is no repetitive behavior in $L$, \emph{go to Step 2}, otherwise do the following.

    Construct the projection $L{\restriction_B}$ and apply Algorithm~$\mathfrak{A_0}(\mathfrak{D})$ to it (with respect to the partition $A=A_1\cup\dots\cup A_k$). 
    
    Let $\Tilde{N}$ be the result high level WF-net over the set $\Tilde{A}$ of sub-process names,  $N_{i_1},\dots, N_{i_j}$ --- result WF-nets for sub-processes with names $\alpha_{i_1},\dots, \alpha_{i_j}$ respectively (for each $\alpha_i$, such that $A_i\subseteq B$, \ie for sub-processes  within the loop body), and let $\beta$  be a new name. Then
    \begin{itemize}
        \item 
        For each $\beta'\in \Tilde{B}$, if there is a transition labeled by $\beta'$ in some of 
        $N_{i_1},\dots, N_{i_j}$ or $\Tilde{N}$, replace this transition with  sub-process $\Loop(\ell_B(\beta'))$, as was done when constructing  a flat WF-net in Subsection~\ref{subsec:HWF}, \ie substitute inner loops in the loop body. Remove $\beta'$ from $\Tilde{A}$, $\Tilde{B}$ and from the domains of $\ell$ and $\ell_B$.
        \item
        Add $\beta$ to $\Tilde{B}$ and extend $\ell_B$ by defining $\ell_B(\beta) = \Tilde{N}$.
    Extend also $\ell$ by defining $\ell(A_{i_1}) = N_{i_1}, \dots, \ell(A_{i_j}) = N_{i_j}$.
        \item
    Replace by $\beta$ all occurrences of activities from $B$ in $L$ and remove stuttering.
        \item 
    If for some $i\in [1\dots k]$, $B\subseteq A_i$, then add $\beta$ to $A_i$ (and respectively to $A$) as one more activity. Otherwise, add $\beta$ to $A$, as well as add  $\{\beta\}$ to partition of $A$. (Thus, $\beta$ is both an activity and a process name, which  should not be confusing.)
    \end{itemize}
    
    \emph{Repeat Step 1}.
     
\item[\textit{Step 2.}] 
Apply Algorithm~$\mathfrak{A_0}(\mathfrak{D})$ to  current log $L$ with respect to the current partition of activities. Let $\Tilde{N}$ be a result high level WF-net.

\item[\textit{Step 3.}] 
While $\Tilde{B}$ is not empty, for each $\beta\in \Tilde{B}$,
 replace a transition labeled by $\beta$ in $\Tilde{N}$ with the sub-process $\Loop(\ell_B(\beta))$, as it is done in Step 1.
    \end{enumerate}
The resulting net $\Tilde{N}$ is a high-level WF-net for the HWF-net constructed by the Algorithm. Its low-level WF-nets are defined by function $\ell$, also built during the Algorithm operation. 
    

\bigskip

\textit{Correctness of the Algorithm~$\mathfrak{A}(\mathfrak{D})$} is justified by the following statement.

\begin{thm}
Let $A$ be a set of activities and $L$ --- a log over $A$. Let also $A=A_1\cup\dots\cup A_k$ be a partition of $A$, 
and $\Tilde{A}= \{\alpha_1, \dots, \alpha_k\}$ --- a set of sub-process names.

If Algorithm $\mathfrak{D}$ for any log $L'$ discovers a WF-net $N'$, such that $N'$ perfectly fits $L'$,
then Algorithm $\mathfrak{A}(\mathfrak{D})$  constructs a HWF-net $\mathcal{N} = (\Tilde{N}, N_1, N_2, \dots, N_k,\ell)$, such that 
$\mathbf{fl}(\mathcal{N})$ perfectly fits $L$.
\end{thm}

\begin{proof}
To prove that the HWF-net built using Algorithm $\mathfrak{A}(\mathfrak{D})$ perfectly fits the input log, provided that Algorithm $\mathfrak{D}$ discovers models with perfect fitness, we use three previously proven assertions.  Namely, 
\begin{itemize}
    \item
    Theorem in \cite{begicheva2017discovering} states that when $\mathfrak{D}$ is an discovery algorithm with perfect fitness, Algorithm~$\mathfrak{A}_0(\mathfrak{D})$ discovers a high-level WF-net, whose refinement perfectly fits the input log without repetitions (\ie  the log of an acyclic process).
    \item
    In \cite{tapia2017discovering} it is proved  that, given a log  $L$, Algorithm $\mathfrak{B}$ correctly finds in $L$ all repetitive components that correspond to supports of T-invariants in the Petri net model for $L$.
    \item
    Proposition~\ref{prop:conn} in Subsection~\ref{subsec:HWF} justify correctness of refining a high-level WF-net by substituting sub-process modules for high-level transitions.
\end{itemize}

With all this, proving the Theorem is straightforward, though quite technical. So, we informally describe the logic of the proof here.

Let 
Algorithm~$\mathfrak{D}$ be a discovery algorithm, which discovers a perfectly fitting WF-net for a given event log.

At each iteration of Step~1, an inner elementary repetitive component in the log is discovered  using Algorithm $\mathfrak{B}$. Activities of this component are activities of an inner loop body, which itself does not  have repetitions. Then a WF-net $N$ for this loop body is correctly discovered using Algorithm~$\mathfrak{A}_0(\mathfrak{D})$, the  loop itself is folded into one high-level activity $\beta$, and $N$ is kept as the value $\ell_B(\beta)$. WF-nets for sub-processes within the body of this loop are also discovered by Algorithm~$\mathfrak{A}_0(\mathfrak{D})$ and accumulated in $\ell$.
If  loop activity $\beta$ is itself within  an upper loop body, then with one more iteration of Step 1, the upper loop $N'$ is discovered, the transition labeled with $\beta$ in it is replaced with $N$, and $N'$ is itself folded into a new activity. 

After processing all loops, the Algorithm proceeds to Step 2, where after reducing all loops to high-level activities, Algorithm~$\mathfrak{A}_0(\mathfrak{D})$ is applied to a log without repetitions.

In Step 3 all transitions labeled with loop activities in a high-level and low-level WF-nets are replaced by WF-nets for these loops, kept in $\ell_B$.

So, we can see that while Algorithm~$\mathfrak{A}_0(\mathfrak{D})$ ensures perfect fitness between acyclic fragments of the model (when loops are folded to transitions), Algorithm $\mathfrak{B}$ ensures correct processing of cyclic behavior, and Proposition~\ref{prop:conn} guarantees that replacing loop activities by corresponding loop WF-nets does not violate fitness, the main algorithm provides systematic log processing and model construction.
\end{proof}

\section{Experimental Evaluation}\label{evaluation}

In this section, we report the main outcomes from a series of experiments conducted to evaluate the algorithm for discovering two-level hierarchical process models from event logs.

To support the automated experimental evaluation, we implemented the hierarchical process discovery algorithm described in the previous section using the open-source library PM4Py \cite{berti2019process}.
The source files of our implementation are published in the open GitHub repository \cite{ExRepo}.
We conducted experiments using two kinds of event logs:
\begin{enumerate}
    \item \emph{Artificial} event logs generated by manually prepared process models;
    \item \emph{Real-life} event logs provided by various information systems.
\end{enumerate}

Event logs are encoded in a standard way as XML-based XES files.

\emph{Conformance checking} is an important part of process mining along with process discovery \cite{Carmona18}.
The main aim of conformance checking is to evaluate the quality of process discovery algorithm by estimating the corresponding quality of discovered process models.
Conformance checking provides four main quality dimensions.
\emph{Fitness} estimates the extent to which a discovered process model can execute traces in an event log. 
A model perfectly fits an event log if it can execute all traces in an event log.
According to Theorem 1, the hierarchical process discovery algorithm yields perfectly fitting process models.
\emph{Precision} evaluates the ratio between the behavior allowed by a process model and not recorded in an event log. 
A model with perfect precision can only execute traces in an event log. The perfect precision limits the use of a process model since an event log represents only a finite ``snapshot'' of all possible process executions. 
Generalization and precision are two dual metrics. 
The fourth metric, \emph{simplicity}, captures the structural complexity of a discovered model.
We improve simplicity by the two-level structure of a discovered process models.

Within the experimental evaluation, we estimated fitness and precision of process models discovered from artificially generated and real-life event logs.
Fitness was estimated using alignments between a process model and an event log as defined in \cite{adriansyah2011conformance}.
Precision was estimated using the complex ETC-align measures proposed in \cite{munoz2010fresh}.
Both measures are values in the interval $[0, 1]$.

\subsection{Discovering HWF-Nets from Artificial Event Logs}
The high-level source for generating artificial low-level event logs was the Petri net earlier shown in Figure \ref{fig:petrinet}.
In this model, we refined its transitions with different sub-processes containing sequential, parallel and cyclic executions of low-level events.
The example of refining the Petri net from Figure \ref{fig:petrinet} is shown in Figure \ref{fig:llclassic}, where we show the corresponding flat representation of an HWF-net.

\begin{figure}[!h]
        \centering
    \includegraphics[width=\textwidth]{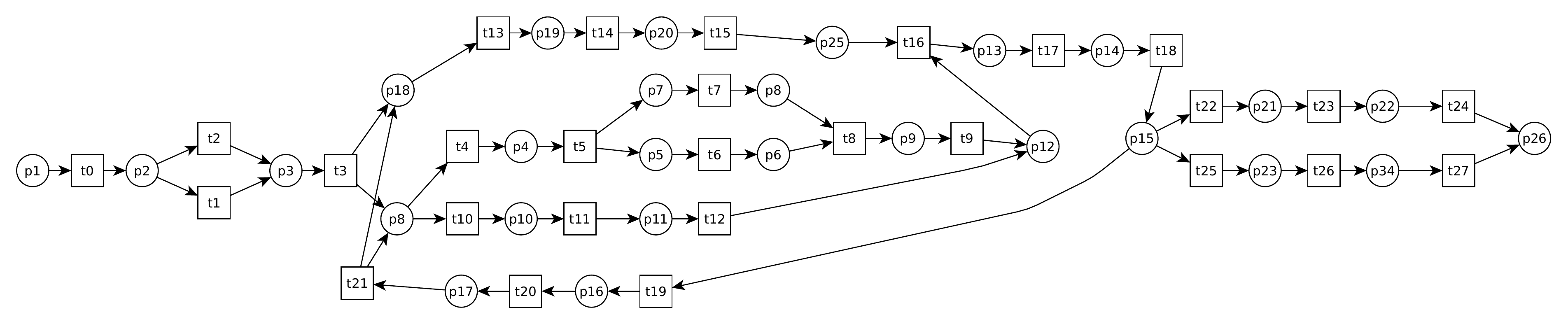}
    \caption{A flat WF-net with generated by refining  the WF-net in~Fig.~\ref{fig:petrinet}}
    \label{fig:llclassic}
\end{figure}

Generation of low-level event logs from the prepared model was implemented using the algorithm presented in \cite{Gena2014}.
Afterwards, we transform a low-level event log into a high-level event log by grouping low-level events into a single high-level event and by extracting the information about cyclic behavior.

The corresponding high-level WF-net discovered from the artificial low-level event log generated from the WF-net shown in \ref{fig:llclassic} is provided in Figure \ref{fig:final_net}.
Intuitively, one can see that this high-level WF-net is rather similar to the original Petri net from Figure \ref{fig:petrinet}.

\begin{figure}[!h]
        \centering
    \includegraphics[width=\textwidth]{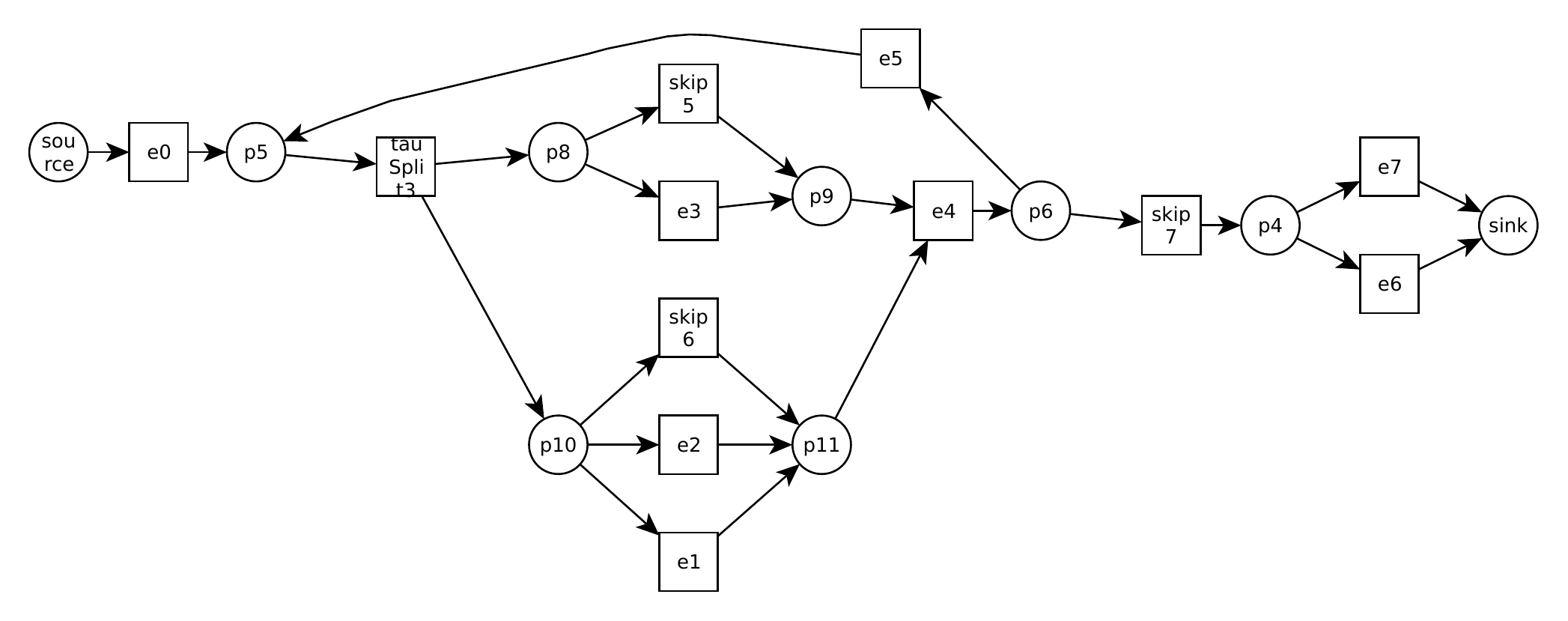}
    \caption{A high-level WF-net discovered from an event log generated by refining the WF-net in Fig.~\ref{fig:llclassic}}
    \label{fig:final_net}
\end{figure}

As for the quality evaluation for the above presented high-level model, we have the following:
\begin{enumerate}
    \item The discovered high-level WF-net perfectly fits the high-level event log obtained from a low-level log, where we identified cycles and grouped activities correspondingly.
    \item The flat WF-net obtained by refining transitions in a discovered high-level WF-net almost perfectly fits ($0.9534$) a low-level log. The main reason for this is the combination of the coverability lack and the straightforward accuracy indicators of the algorithm.
    \item Precision of the flat WF-net is $0.3729$, which is the result of identifying independent sub-processes generalizing the behavior recorded in an initial low-level event log.
\end{enumerate}

Other examples of process models that were used for the artificial event log generation are also provided in the main repository \cite{ExRepo}. 

\subsection{Discovering HWF-Nets from Real-Life Event Logs}

We used two real-life event logs provided by \emph{BPI} (Business Process Intelligence) \emph{Challenge} 2015 and 2017 \cite{bpiData}.
These event logs were also enriched with the additional statistical information about flat process models.

The \emph{BPI Challenge 2015} event log was provided by five Dutch municipalities.
The cases in this event log contain information on the main application and objection procedures in various stages.
A flat low-level WF-net for case $f1$ discovered using the Inductive miner is shown in Figure \ref{fig:initial_net_BPIC151}.
It is easy to see that this model is absolutely inappropriate for the visual analysis.

\begin{figure}[!h]
        \centering
    \includegraphics[width=\textwidth]{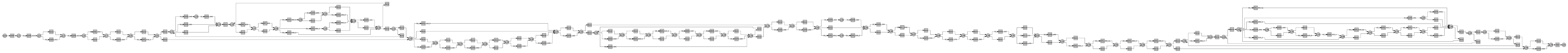}
    \caption{A flat WF-net discovered from BPI Challenge 2015 event log}
    \label{fig:initial_net_BPIC151}
\end{figure}

The code of each event in the \emph{BPI Challenge 2015} event log consists of three parts: two digits, a variable number of characters, and three digits more. 
From the event log description we know that the first two digits and the characters indicate the sub-process the event belongs to, which allows us to assume an option of identifying the sub-processes. 
We used the first two parts of the event name to create the mapping between low-level events and sub-proces names. 
After applying our hierarchical process discovery algorithm in combination with the Inductive Miner, we obtained a high-level model presented in Figure \ref{fig:hl_net_BPIC151} that is far more comprehensible than the flat model mainly because of its size.

\begin{figure}[!h]
        \centering
    \includegraphics[width=0.8\textwidth]{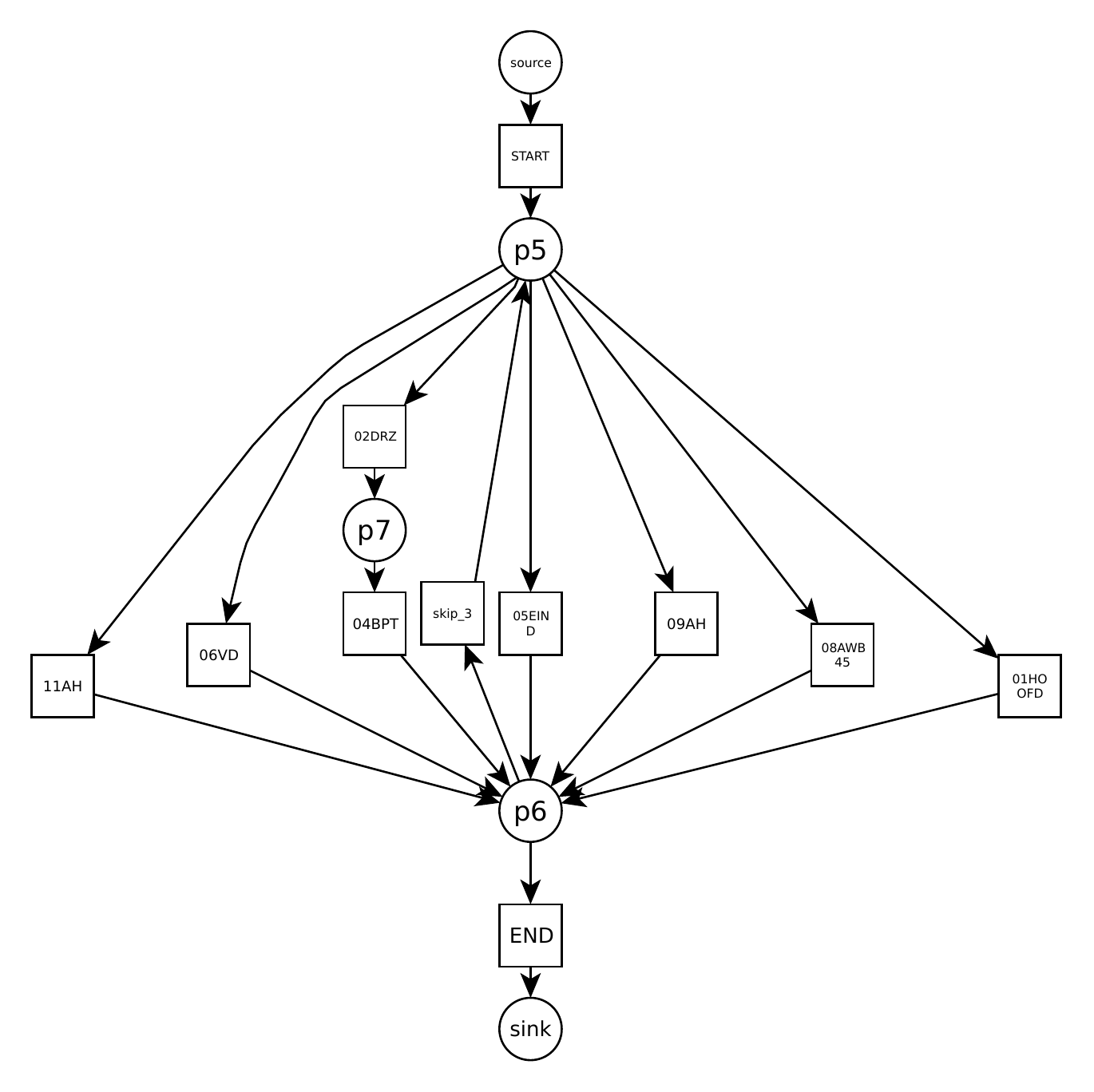}
    \caption{A high-level WF-net discovered from the 'BPI Challenge 2015' event log}
    \label{fig:hl_net_BPIC151}
\end{figure}

The \emph{BPI Challenge 2017} event log pertains to a loan application process of a Dutch financial institute. 
The data contains all applications filed trough an online system from 2016 till February of 2017. 
Here, as a base for mapping low-level events to sub-proces names, we used the mark of the event type in its name --- application, offer or workflow.
Thus, a mapping could be based on various features of event data dependint on the expert's needs. 
The flat model for this data is presented in Figure \ref{fig:initial_net_BPIC17}, which is also difficult to read because of  its purely sequential representation.
\medskip

\begin{figure}
        \centering
    \includegraphics[width=\textwidth]{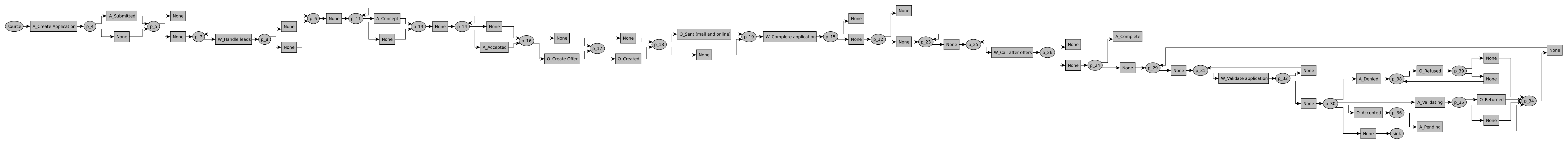}
    \caption{A flat WF-net discovered from the 'BPI Challenge 2017' event log}
    \label{fig:initial_net_BPIC17}
\end{figure}

Applying the  principle of mapping low-level events in the \emph{BPI Challenge 2017} event log  described above, we obtained the high-level WF-net shown in Fig.~\ref{fig:hl_net_BPIC17}, which clearly demonstrates sub-processes (if necessary, they can be expanded) and their order.
\bigskip

\begin{figure}
        \centering
    \includegraphics[width=0.34\textwidth]{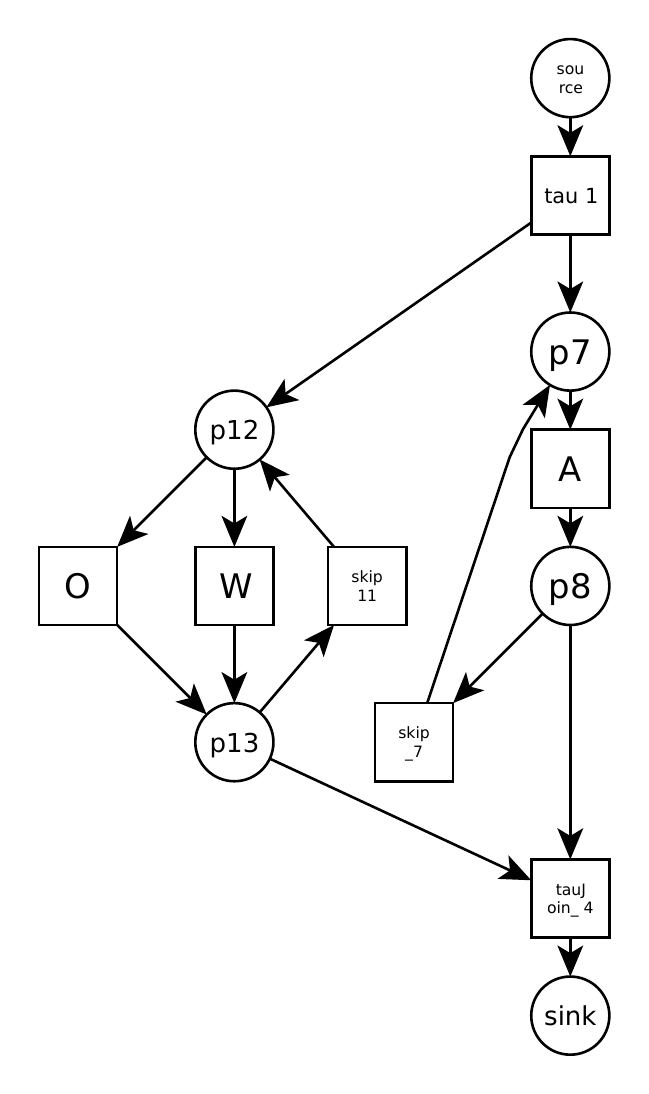}
    \caption{A high-level WF-net discovered from the 'BPI Challenge 2017' event log}
    \label{fig:hl_net_BPIC17}
\end{figure}

Table \ref{tab:rlq} shows fitness and precision evaluation for flat and high-level WF-nets discovered from real-life \emph{BPI Challenge 2015} and \emph{2017} event logs.

\medskip

\begin{table}[!h]
	\centering
	\caption{Comparing metrics for  flat and high-level WF-nets discovered from 'BPI Challenge' event logs}\label{tab:rlq}
	\begin{tabular}{|l|l|l|l|l|l|}
		\hline
        \multirow{2}{*}{Event log} & \multicolumn{3}{|c|}{High-level WF-net} & \multicolumn{2}{|c|}{Flat WF-net} \\ \cline{2-6}
        & Fitness 1 & Fitness 2& Precision & Fitness & Precision \\ \hline
        BPI Challenge 2015 & 0.9122 & 1 & 0.5835 & 0.9700 & 0.5700 \\ \hline
        BPI Challenge 2017 & 0.9393 & 1 & 0.3898 & 0.9800 & 0.7000 \\ \hline
	\end{tabular}
\end{table}

\emph{Fitness 1} shows the fitness evaluation between the flat WF-net constructed from the high-level WF-net by refining transitions with low-level subprocesses.
\emph{Fitness 2} shows the fitness evaluation between the high-level WF-net and an event log with low-level events grouped into sub-processes.
This confirms the formal correctness results of the hierarchical process discovery algorithm.
Similar to the experimental results for artificial event logs, here we also observe a decrease in the precision for the identification of sub-processes, therefore, generalizing traces in an initial low-level event log.

\section{Conclusion and Future Work}\label{sec7}

In this research, we provide a new process discovery technique for solving the problem of discovering a hierarchical WF-net model from a low-level event log, based on ``folding'' sub-processes into high-level transitions according to event clustering. Unlike the previous solutions, we allow cycles and concurrency in process behavior. 

We prove that the proposed technique makes it possible to obtain hierarchical models, which fit event logs perfectly. The technique was also evaluated on artificial and real event logs. Experiments show that fitness and precision of obtained hierarchical models are almost the same as for the standard ``flat'' case, while hierarchical models are much more compact, more readable and more visual.

To implement our algorithm and check it on real data we used \emph{Python} and one of the most convenient instruments for process mining at the moment --- the \emph{PM4Py} \cite{berti2019process}. The implementation  is provided in the public \emph{GitHub} repository \cite{ExRepo}.

In further research, we plan to develop and evaluate various event clustering methods for automatic discovery of hierarchical models.

 \bibliographystyle{elsarticle-num} 
 \bibliography{references}

\begin{thebibliography}{10}
\expandafter\ifx\csname url\endcsname\relax
  \def\url#1{\texttt{#1}}\fi
\expandafter\ifx\csname urlprefix\endcsname\relax\def\urlprefix{URL }\fi
\expandafter\ifx\csname href\endcsname\relax
  \def\href#1#2{#2} \def\path#1{#1}\fi

\bibitem{augusto2018automated}
A.~Augusto, R.~Conforti, M.~Dumas, M.~La~Rosa, F.~M. Maggi, A.~Marrella,
  M.~Mecella, A.~Soo, Automated discovery of process models from event logs:
  review and benchmark, IEEE Transactions on Knowledge and Data Engineering
  31~(4) (2018) 686--705.

\bibitem{wfnets00}
W.~van~der Aalst, Workflow verification: Finding control-flow errors using
  petri-net-based techniques, in: W.~van~der Aalst, J.~Desel, A.~Oberweis
  (Eds.), Business Process Management: Models, Techniques, and Empirical
  Studies, Vol. 1806 of Lecture Notes in Computer Science, Springer,
  Heidelberg, 2000, pp. 161--183.

\bibitem{begicheva2017discovering}
A.~K. Begicheva, I.~A. Lomazova, Discovering high-level process models from
  event logs, Modeling and Analysis of Information Systems 24~(2) (2017)
  125--140.

\bibitem{van2021event}
S.~J. van Zelst, F.~Mannhardt, M.~de~Leoni, A.~Koschmider, Event abstraction in
  process mining: literature review and taxonomy, Granular Computing 6~(3)
  (2021) 719--736.

\bibitem{maneschijn2022methodology}
D.~G. Maneschijn, R.~H. Bemthuis, F.~A. Bukhsh, M.-E. Iacob, A methodology for
  aligning process model abstraction levels and stakeholder needs, in:
  Proceedings of the 24th International Conference on Enterprise Information
  Systems - Volume 1: ICEIS, 2022, pp. 137--147.

\bibitem{Mannhardt2016}
F.~Mannhardt, M.~de~Leoni, H.~Reijers, W.~van~der Aalst, P.~Toussaint, From
  low-level events to activities -- a pattern-based approach, in: Business
  Process Management. BPM 2016, Vol. 9850 of Lecture Notes in Computer Science,
  Springer, Cham, 2016, pp. 125--141.

\bibitem{tax2016event}
N.~Tax, N.~Sidorova, R.~Haakma, W.~van~der Aalst, Event abstraction for process
  mining using supervised learning techniques, in: Proceedings of SAI
  Intelligent Systems Conference (IntelliSys) 2016, Vol.~15 of Lecture Notes in
  Networks and Systems, Springer, Cham, 2018, pp. 161--170.

\bibitem{leemans2020using}
S.~J. Leemans, K.~Goel, S.~J. van Zelst, Using multi-level information in
  hierarchical process mining: Balancing behavioural quality and model
  complexity, in: 2020 2nd International Conference on Process Mining (ICPM),
  IEEE, 2020, pp. 137--144.

\bibitem{Smirnov2012}
S.~Smirnov, H.~Reijers, M.~Weske, T.~Nugteren, Business process model
  abstraction: a definition, catalog, and survey, Distributed and Parallel
  Databases 30 (2012) 63--99.

\bibitem{gunther2007fuzzy}
C.~W. G{\"u}nther, W.~M. Van Der~Aalst, Fuzzy mining--adaptive process
  simplification based on multi-perspective metrics, in: International
  conference on business process management, Springer, 2007, pp. 328--343.

\bibitem{Reisig13}
W.~Reisig, Understanding Petri Nets: Modeling Techniques, Analysis Methods,
  Case Studies, Springer, Heidelberg, 2013.

\bibitem{jensen2009coloured}
K.~Jensen, L.~Kristensen, Coloured Petri nets: modelling and validation of
  concurrent systems, Springer, 2009.

\bibitem{Leemans13}
S.~Leemans, D.~Fahland, W.~van~der Aalst, Discovering block-structured process
  models from event logs -- a constructive approach, in: Application and Theory
  of Petri Nets and Concurrency, Vol. 7927 of Lecture Notes in Computer
  Science, Springer, Heidelberg, 2013, pp. 311--329.

\bibitem{greco2008mining}
G.~Greco, A.~Guzzo, L.~Pontieri, Mining taxonomies of process models, Data \&
  Knowledge Engineering 67~(1) (2008) 74--102.

\bibitem{li2010mining}
J.~Li, R.~Bose, W.~van~der Aalst, Mining context-dependent and interactive
  business process maps using execution patterns, in: Business Process
  Management Workshops. BPM 2010, Vol.~66 of Lecture Notes in Business
  Information Processing, Springer Heidelberg, 2010, pp. 109--121.

\bibitem{lu2020discovering}
X.~Lu, A.~Gal, H.~A. Reijers, Discovering hierarchical processes using flexible
  activity trees for event abstraction, in: 2020 2nd International Conference
  on Process Mining (ICPM), IEEE, 2020, pp. 145--152.

\bibitem{Fstr07}
W.~{van der Aalst}, C.~{Gunther}, Finding structure in unstructured processes:
  The case for process mining, in: Seventh International Conference on
  Application of Concurrency to System Design (ACSD 2007), IEEE, 2007, pp.
  3--12.
\newblock \href {https://doi.org/10.1109/ACSD.2007.50}
  {\path{doi:10.1109/ACSD.2007.50}}.

\bibitem{MpPM14}
J.~De~Smedt, J.~De~Weerdt, J.~Vanthienen, Multi-paradigm process mining:
  Retrieving better models by combining rules and sequences, in: On the Move to
  Meaningful Internet Systems: OTM 2014 Conferences, Vol. 8841 of Lecture Notes
  in Computer Science, Springer, Heidelberg, 2014, pp. 446--453.
\newblock \href {https://doi.org/10.1007/978-3-662-45563-0\_26}
  {\path{doi:10.1007/978-3-662-45563-0\_26}}.

\bibitem{Pedro16}
J.~de~San~Pedro, J.~Cortadella, Mining structured petri nets for the
  visualization of process behavior, in: Proceedings of the 31st Annual ACM
  Symposium on Applied Computing, ACM, 2016, p. 839–846.
\newblock \href {https://doi.org/10.1145/2851613.2851645}
  {\path{doi:10.1145/2851613.2851645}}.

\bibitem{LocEv15}
W.~van~der Aalst, A.~Kalenkova, V.~Rubin, E.~Verbeek, Process discovery using
  localized events, in: R.~Devillers, A.~Valmari (Eds.), Application and Theory
  of Petri Nets and Concurrency, Vol. 9115 of Lecture Notes in Computer
  Science, Springer, Cham, 2015, pp. 287--308.
\newblock \href {https://doi.org/10.1007/978-3-319-19488-2\_15}
  {\path{doi:10.1007/978-3-319-19488-2\_15}}.

\bibitem{Kalenkova14-1}
A.~Kalenkova, I.~Lomazova, Discovery of cancellation regions within process
  mining techniques, Fundamenta Informaticae 133 (2014) 197--209.
\newblock \href {https://doi.org/10.3233/FI-2014-1071}
  {\path{doi:10.3233/FI-2014-1071}}.

\bibitem{Kalenkova14}
A.~Kalenkova, I.~Lomazova, W.~van~der Aalst, Process model discovery: A method
  based on transition system decomposition, in: G.~Ciardo, E.~Kindler (Eds.),
  Application and Theory of Petri Nets and Concurrency, Vol. 8489 of Lecture
  Notes in Computer Science, Springer, Cham, 2014, pp. 71--90.
\newblock \href {https://doi.org/10.1007/978-3-319-07734-5\_5}
  {\path{doi:10.1007/978-3-319-07734-5\_5}}.

\bibitem{alvarez2020identifying}
Y.~Alvarez-P{\'e}rez, E.~L{\'o}pez-Mellado, Identifying petri nets with silent
  transitions by event traces classification, IFAC-PapersOnLine 53~(4) (2020)
  199--204.

\bibitem{van2002workflow}
W.~M. Van Der~Aalst, Workflow verification: Finding control-flow errors using
  petri-net-based techniques, in: Business process management: models,
  techniques, and empirical studies, Springer, 2002, pp. 161--183.

\bibitem{Aalst11}
W.~Van~der Aalst, Process mining: Discovery, conformance and enhancement of
  business processes (2011).

\bibitem{tapia2017discovering}
T.~Tapia-Flores, E.~L{\'o}pez-Mellado, A.~P. Estrada-Vargas, J.-J. Lesage,
  Discovering petri net models of discrete-event processes by computing
  t-invariants, IEEE Transactions on Automation Science and Engineering 15~(3)
  (2017) 992--1003.

\bibitem{berti2019process}
A.~Berti, S.~Van~Zelst, W.~van~der Aalst, Process mining for python (pm4py):
  bridging the gap between process- and data science, in: Proceedings of the
  ICPM Demo Track 2019, Vol. 2374 of CEUR Workshop Proceedings, CEUR-WS.org,
  2019, pp. 13--16.

\bibitem{ExRepo}
A.~Begicheva, \href{\url{github.com/gingerabsurdity/hldiscovery}}{Hierarchical
  process model discovery -- hldiscovery.} (2022).
\newline\urlprefix\url{\url{github.com/gingerabsurdity/hldiscovery}}

\bibitem{Carmona18}
J.~Carmona, B.~van Dongen, A.~Solti, M.~Weidlich, Conformance Checking:
  Relating Processes and Models, Springer, Cham, 2018.

\bibitem{adriansyah2011conformance}
A.~Adriansyah, B.~F. van Dongen, W.~M. van~der Aalst, Conformance checking
  using cost-based fitness analysis, in: 2011 ieee 15th international
  enterprise distributed object computing conference, IEEE, 2011, pp. 55--64.

\bibitem{munoz2010fresh}
J.~Munoz-Gama, J.~Carmona, A fresh look at precision in process conformance,
  in: International Conference on Business Process Management, Springer, 2010,
  pp. 211--226.

\bibitem{Gena2014}
I.~Shugurov, A.~Mitsyuk, Generation of a set of event logs with noise, in:
  Proceedings of the 8th Spring/Summer Young Researchers Colloquium on Software
  Engineering (SYRCoSE 2014), 2014, pp. 88--95.

\bibitem{bpiData}
A.~Augusto, R.~Conforti, M.~M. Dumas, M.~L. Rosa, F.~Maggi, A.~Marrella,
  M.~Mecella, A.~Soo,
  \href{\url{https://data.4tu.nl/articles/dataset/Data_underlying_the_paper_Automated_Discovery_of_Process_Models_from_Event_Logs_Review_and_Benchmark/12712727}}{Data
  underlying the paper: Automated discovery of process models from event logs:
  Review and benchmark} (6 2019).
\newblock \href
  {https://doi.org/10.4121/uuid:adc42403-9a38-48dc-9f0a-a0a49bfb6371}
  {\path{doi:10.4121/uuid:adc42403-9a38-48dc-9f0a-a0a49bfb6371}}.
\newline\urlprefix\url{\url{https://data.4tu.nl/articles/dataset/Data_underlying_the_paper_Automated_Discovery_of_Process_Models_from_Event_Logs_Review_and_Benchmark/12712727}}

\end{thebibliography}





\end{document}